\documentclass[11pt]{article}

\usepackage[T1]{fontenc}
\usepackage[utf8]{inputenc}
\usepackage{lmodern}
\usepackage[a4paper,margin=2.5cm]{geometry}
\usepackage{amsmath,amssymb,amsthm}
\usepackage{mathtools}
\usepackage{graphicx}
\usepackage{booktabs}
\usepackage{enumitem}
\usepackage{microtype}
\usepackage[hidelinks]{hyperref}

\newtheorem{theorem}{Theorem}
\newtheorem{proposition}{Proposition}
\theoremstyle{remark}

\DeclareMathOperator{\im}{im}
\DeclareMathOperator{\rank}{rank}
\DeclareMathOperator{\tr}{tr}
\DeclareMathOperator{\nnz}{nnz}
\DeclareMathOperator{\Hol}{Hol}
\DeclareMathOperator{\wrap}{wrap}
\DeclareMathOperator{\Calibrate}{Calibrate}
\newcommand{\T}{\top}
\newcommand{\R}{\mathbb{R}}
\newcommand{\Z}{\mathbb{Z}}
\newcommand{\F}{\mathbb{F}}
\newcommand{\cA}{\mathcal{A}}
\newcommand{\cL}{\mathcal{L}}
\newcommand{\norm}[1]{\lVert #1\rVert}

\title{HodgePhaseNet: Discrete Hodge Decomposition of Semantic Flows\\ in Transformer Attention}
\author{Kiyotaka Kasubuchi \and Kazuo Fukiya}
\date{September 2026}

\begin{document}
\maketitle

\begin{abstract}
WavePhaseNet introduced frequency-resolved semantic representations, QuantumPhaseNet introduced gauge-covariant phase connections, and ModularPhaseNet made a finite cyclic phase channel and its cycle holonomy computable in an ordinary Transformer. These constructions describe complementary aspects of semantic organization, yet they do not by themselves separate an observed collection of pairwise relations into a globally consistent direction, circulation around local compositions, and circulation that persists around a larger unfilled cycle. HodgePhaseNet addresses this separation by representing semantic transitions as an antisymmetric real-valued 1-cochain on a finite oriented simplicial complex. With vertex--edge incidence matrix $B_1$ and edge--face incidence matrix $B_2$, every such field has a unique orthogonal decomposition $\Omega = B_1^{\T}\Phi + B_2\Psi + H$, where $B_1H=0$ and $B_2^{\T}H=0$. The three component \emph{fields} are unique even when their potentials are not. We give explicit projectors, a sparse implementation, integrability and gauge results, a verified numerical example, and a redesigned attention interface. The attention bias is defined on allowed ordered query--key pairs, remains consistent under arbitrary choices of edge orientation, and preserves a decoder's causal information boundary by constructing its complex from the visible prefix. Gauge invariance of the attention output is conditional on using invariant or appropriately compensated features. A simple difference of vertex embeddings has no coexact or harmonic component; a genuine relation encoder is therefore indispensable. Local circulation is not automatically contradiction, and harmonic flow is not automatically missing knowledge. These interpretations and any reduction in factual error require controlled empirical evaluation. This paper presents a mathematical and architectural proposal, not external natural-language performance results.
\end{abstract}


\section{Introduction}

\subsection{From the three PhaseNet constructions to an edge field}

The original Transformer computes a similarity between a query and a key and uses a softmax to mix value vectors~\cite{vaswani2017}. It can learn complex relations implicitly, but the dot product does not explicitly distinguish a direction inferred from a common potential, a local three-way circulation, and a cycle that cannot be reduced to any chosen local face. This distinction matters when the intended output is a hierarchy, a chain of reasoning, or a claim supported by several paths of evidence.

WavePhaseNet~\cite{wavephasenet} starts with the \emph{scale} of a representation. Its discrete Fourier transform along a sequence separates frequency bands; low and high bands can be used as candidate features of longer-range intent and local expression. The assignment of semantic roles to these bands remains an empirical interpretation. Its treatment of overlapping windows also invokes graph-based consistency and harmonic projection. The present work does not claim to introduce harmonic ideas into the PhaseNet family for the first time. It moves the object of decomposition to an independently constructed, directed \textbf{edge relation field} and separates that field into three mutually orthogonal parts.

QuantumPhaseNet~\cite{quantumphasenet} adds the \emph{comparison rule} between local complex phase representations. When phases are changed by local gauge transformations, their bare differences depend on the chosen coordinates; a connection permits covariant comparison and describes phase consistency. Its reported prototype results concern a built-in synthetic setting for a classical quantum-inspired model. A covariant connection supplies a natural source of edge information, but its mere availability does not say which part of an entire edge field is a potential difference, which is explained by the boundaries of selected faces, and which survives as a global cycle.

ModularPhaseNet~\cite{modularphasenet} supplies a \emph{finite implementation and an exact group-valued test}. Its auxiliary phase lies in a cyclic subgroup of a finite field's multiplicative group, while the Transformer's principal hidden states remain real. Products of relative group elements around cycles yield gauge-invariant holonomy. A zero-holonomy criterion answers whether a finite group connection is globally integrable, given all cycles; it does not supply a real orthogonal allocation of edge energy among potential, local, and global contributions. Nor may a modular residue be identified with a real vector without choosing a lift and accepting its branch dependence.

The missing step is therefore well defined. Frequency features specify scale; covariant phase features specify how to compare neighboring states; finite group features give an exact discrete consistency signal. We use these as inputs to a relation encoder and then ask how its \textbf{real antisymmetric edge field} decomposes over a declared set of vertices, edges, and composable faces. The identity that a boundary has no boundary makes this question an elementary finite-dimensional Hodge problem. Its answer can be returned to attention as structured features. This is the precise bridge from the three preceding papers to HodgePhaseNet. Figure~\ref{fig:roles} summarizes these complementary roles and the new edge-field step.

\begin{figure}[t]
  \centering
  \includegraphics[width=0.92\linewidth]{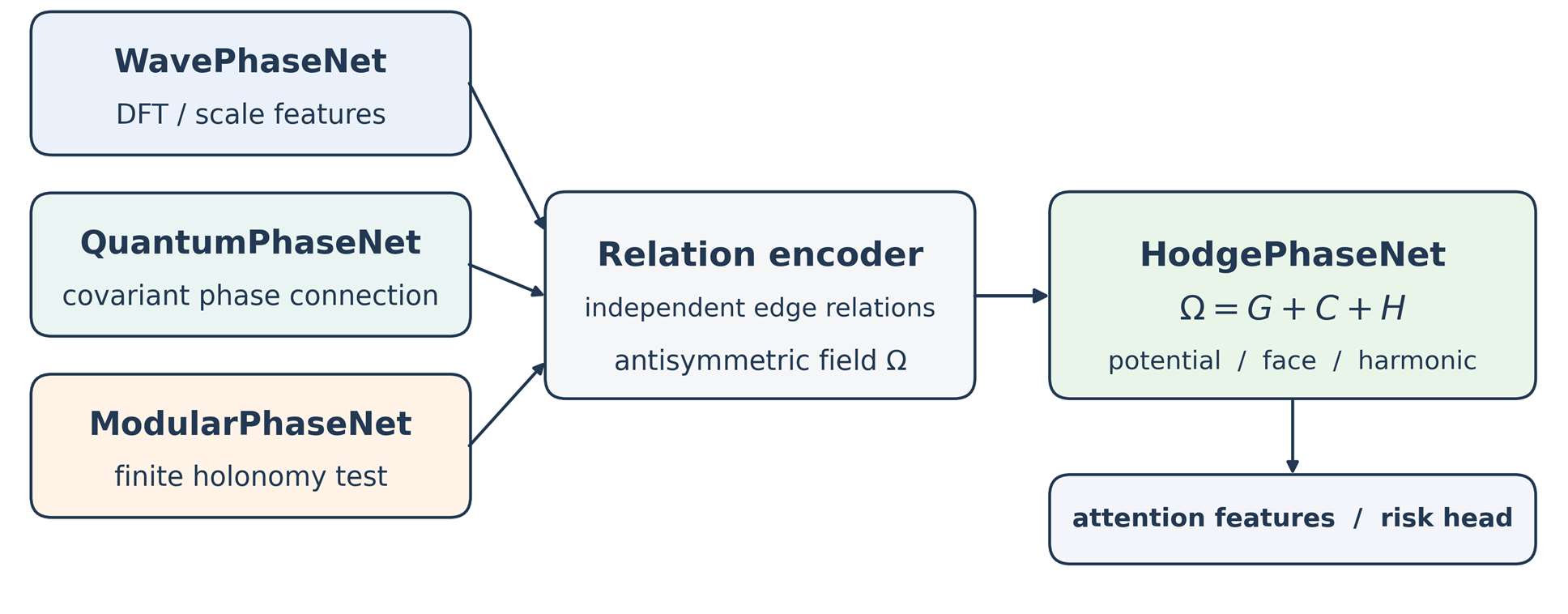}
  \caption{Complementary roles of the earlier PhaseNet methods. Their features feed a separately learned real edge field; the Hodge layer separates that field into gradient, local coexact, and global harmonic components before returning structured features to attention.}
  \label{fig:roles}
\end{figure}

\subsection{Scope and contributions}

The paper establishes: (i) an edge-level mathematical model with a nondegenerate relation encoder; (ii) an exact three-term decomposition, projectors, weighted extension, and a worked example; (iii) a gauge- and orientation-aware integration into self-attention, including causal decoding; and (iv) a falsifiable evaluation protocol for semantic interpretation, factuality, calibration, and cost. The theorems concern a fixed finite complex and specified inner product. They do not assert that a learned complex is a uniquely correct model of language, or that algebraic circulation is an error label.

\section{Background and related work}

\subsection{Frequency, phase, and finite cyclic holonomy}

For token representations $X=(x_0,\dots,x_{N-1})$, WavePhaseNet uses the sequence DFT
\begin{equation}
  \hat{x}_k = \sum_{n=0}^{N-1} x_n \exp\!\left(-\frac{2\pi i k n}{N}\right).
\end{equation}
The DFT and its inverse are exact finite-dimensional operations, with a Parseval energy identity after the usual normalization. A low-frequency band is a possible descriptor of longer-range semantic organization, while local residuals may encode fine detail~\cite{wavephasenet}. These quantities become relation-encoder features, not an asserted proof that a particular frequency equals a particular concept.

QuantumPhaseNet represents a context-dependent state schematically as $\psi = A e^{i\phi} u$ and introduces a connection for parallel comparison under a local phase change~\cite{quantumphasenet}. For an oriented edge $e\colon i\to j$, we adopt the real lift convention
\begin{equation}\label{eq:gauge}
  a_e \mapsto a'_e = a_e + (\chi_j - \chi_i).
\end{equation}
The associated covariant quantities must transform consistently; for example, if $\theta_i\mapsto\theta_i+\chi_i$, then $a_e-(\theta_j-\theta_i)$ is invariant in a fixed real phase chart. A real lift is only local when phase wrapping occurs.

ModularPhaseNet represents the auxiliary phase by an exponent $a\in\Z/q\Z$, optionally as $g^a$ in a subgroup of order $q\mid(p-1)$ in $\F_p^{\times}$~\cite{modularphasenet}. For an oriented cycle $\gamma$, its additive holonomy is
\begin{equation}\label{eq:hol}
  \Hol_q(\gamma) = \sum_{e\in\gamma} s_{\gamma e}\, a_e \pmod q,
\end{equation}
where $s_{\gamma e}$ records relative orientation. The product representation is equivalent. On a connected graph, vanishing holonomy on a cycle basis is necessary and sufficient for an edge assignment to be the difference of vertex exponents, up to a constant. This algebraic fact and a real Hodge energy decomposition answer different questions.

\subsection{Discrete Hodge theory}

Combinatorial Hodge theory~\cite{eckmann1944} splits finite cochains into exact, coexact, and harmonic subspaces. Edge-flow applications to ranking~\cite{jiang2011}, higher-order random walks~\cite{schaub2020}, and simplicial signal processing~\cite{barbarossa2020} show the mathematical utility of this distinction. Simplicial neural networks, Hodgelets, and related edge models~\cite{ebli2020,roddenberry2022,yang2023,bodnar2021} motivate learnable higher-order interactions. HodgePhaseNet contributes a task-specific route from the PhaseNet feature families to a semantic edge field, and from its decomposition back to the exact ordered query--key scores of a Transformer. Its empirical advantage over these alternatives remains an open question.

\section{Semantic complex and discrete operators}

\subsection{Oriented complex}

Let $K=(V,E,F)$ be a finite two-dimensional simplicial complex with $n_0$ vertices, $n_1$ edges, and $n_2$ triangular faces. A vertex may stand for a token, sentence, claim, entity, or concept; the chosen granularity is fixed for a given layer. An edge is a candidate composable relation; a face means that the three boundary relations are judged locally composable by an explicit task rule. A triangle of edges need not be filled automatically. Removing all faces removes the coexact subspace; filling every possible triangle may remove a global cycle that the application intended to retain.

Choose an arbitrary orientation for every edge and face. The incidence matrices are $B_1\in\R^{n_0\times n_1}$ and $B_2\in\R^{n_1\times n_2}$. Column $e\colon i\to j$ of $B_1$ has $-1$ at $i$ and $+1$ at $j$; $B_2$ records the oriented boundary of each face. The chain condition is
\begin{equation}\label{eq:chain}
  B_1B_2 = 0.
\end{equation}
The orientations are coordinate choices, not predictions about the truth or temporal direction of a relation.

\subsection{Cochains and Laplacians}

With $r$ real channels, $C^k(K;\R^r)=\R^{n_k\times r}$ for $k=0,1,2$. A vertex potential $\Phi\in C^0$, an edge field $\Omega\in C^1$, and a face potential $\Psi\in C^2$ have the unweighted Frobenius inner product. The coboundaries and adjoints are
\begin{equation}
  d_0 = B_1^{\T},\quad d_1 = B_2^{\T},\quad \delta_1 = B_1,\quad \delta_2 = B_2,\quad d_1d_0 = 0.
\end{equation}
Thus $B_2^{\T}B_1^{\T}\Phi=0$: potential differences have zero sum around every \emph{filled} face. The 1-Hodge Laplacian and its energy are
\begin{equation}\label{eq:L1}
  L_1 = B_1^{\T}B_1 + B_2B_2^{\T},\qquad
  \langle\Omega, L_1\Omega\rangle = \norm{B_1\Omega}_F^2 + \norm{B_2^{\T}\Omega}_F^2.
\end{equation}
Consequently $\ker L_1=\ker B_1\cap\ker B_2^{\T}$. The two terms in~\eqref{eq:L1} measure vertex divergence and circulation on the declared faces. They are algebraic quantities, not labels of factual accuracy.

\subsection{Constructing a nontrivial edge field}\label{sec:encoder}

At layer $\ell$, let $x_i^{(\ell)}$ be a vertex state. A relation encoder evaluates both orders of an edge with correspondingly reversed relation metadata:
\begin{equation}\label{eq:encoder}
  R_{ij} = R_\vartheta\bigl(x_i,x_j,r_{ij},c_{ij},f^W_{ij},f^Q_{ij},f^M_{ij},f^E_{ij}\bigr),\qquad
  \Omega_{i\to j} = \frac{R_{ij}-R_{ji}}{2}.
\end{equation}
This guarantees $\Omega_{j\to i}=-\Omega_{i\to j}$. Here $f^W$ may contain band energies and reconstructions; $f^Q$ may contain covariant phase residuals and a real connection; $f^M$ may contain cyclic group relations and nearby holonomy; and $f^E$ may contain retrieved evidence and relation type. These are optional, measurable inputs. When a gauge result is required, parameterize a connection channel as $\Omega=a+\tilde{R}$ with $\tilde{R}$ gauge invariant and orientation equivariant, so~\eqref{eq:gauge} implies $\Omega'=\Omega+B_1^{\T}\chi$. An arbitrary unconstrained encoder does not automatically have this transformation law.

\begin{proposition}[degeneracy of simple differences]\label{prop:degenerate}
If $X\in C^0(K;\R^r)$ and $\Omega=B_1^{\T}X$, its coexact and harmonic projections are zero.
\end{proposition}
\begin{proof}
$\Omega\in\im B_1^{\T}$, which is orthogonal to $\im B_2$ by~\eqref{eq:chain} and to $\ker L_1$ by definition. Thus both other projections vanish.
\end{proof}

This is a substantive architectural constraint. A complex built from vertex differences alone cannot display the intended local or global circulation, irrespective of how its downstream Hodge layer is tuned.

\section{Orthogonal decomposition and computable projectors}

\subsection{Main theorem}

\begin{theorem}[discrete semantic Hodge decomposition]\label{thm:hodge}
For fixed $K$, every $\Omega\in C^1(K;\R^r)$ has a unique orthogonal decomposition into component fields
\begin{equation}\label{eq:hodge}
  \Omega = G + C + H = B_1^{\T}\Phi + B_2\Psi + H,\qquad B_1H=0,\quad B_2^{\T}H=0.
\end{equation}
The potentials $\Phi$ and $\Psi$ need not be unique; $G$, $C$, $H$ are unique.
\end{theorem}
\begin{proof}
From~\eqref{eq:chain}, $\langle B_1^{\T}\Phi, B_2\Psi\rangle_F=\tr(\Phi^{\T}B_1B_2\Psi)=0$. The orthogonal complement of $\im B_1^{\T}\oplus\im B_2$ is $\ker B_1\cap\ker B_2^{\T}=\ker L_1$. Finite-dimensional orthogonal projection gives existence and uniqueness of the fields.
\end{proof}

Writing $A^{\dagger}$ for the Moore--Penrose pseudoinverse, the projectors are
\begin{equation}\label{eq:proj}
  P_G = B_1^{\T}(B_1B_1^{\T})^{\dagger}B_1,\qquad
  P_C = B_2(B_2^{\T}B_2)^{\dagger}B_2^{\T},\qquad
  P_H = I - P_G - P_C.
\end{equation}
We compute $G=P_G\Omega$, $C=P_C\Omega$, and $H=P_H\Omega$. For large sparse complexes, do not form the dense projectors. Instead solve, column by column,
\begin{equation}\label{eq:solve}
  (B_1B_1^{\T})\Phi = B_1\Omega,\qquad (B_2^{\T}B_2)\Psi = B_2^{\T}\Omega,\qquad H = \Omega - B_1^{\T}\Phi - B_2\Psi.
\end{equation}
Set the mean of $\Phi$ to zero separately on each connected component, or pin one vertex in each component. A minimum-norm solve handles a nontrivial kernel of the face system. Preconditioned conjugate gradients on the normal systems, LSQR, or an appropriate sparse factorization are possible; tolerances must be recorded. Approximate solves have residual errors and should not be described as exact orthogonal projections.

\subsection{Energy, topology, and integrability}

The components obey
\begin{equation}\label{eq:energy}
  \norm{\Omega}_F^2 = \norm{G}_F^2 + \norm{C}_F^2 + \norm{H}_F^2,\qquad
  \rho_s = \frac{\norm{s}_F^2}{\norm{\Omega}_F^2},\qquad
  \sum_{s\in\{G,C,H\}}\rho_s = 1\quad(\Omega\neq0).
\end{equation}
For real coefficients, $\dim\ker L_1=\beta_1(K)$, the first Betti number. A harmonic \textbf{subspace} exists when topology permits it, but a given field can have zero projection onto it. Neither $\beta_1>0$ nor $\norm{H}>0$ alone proves a missing fact.

\begin{theorem}[global potential criterion]\label{thm:potential}
On a connected $K$, $\Omega=B_1^{\T}\Phi$ for a vertex potential unique up to an additive constant if and only if $C=H=0$.
\end{theorem}
\begin{proof}
This is the uniqueness of the three-term orthogonal decomposition. If two potentials yield the same field, their difference belongs to $\ker B_1^{\T}$ and is constant on a connected graph.
\end{proof}

The local condition $B_2^{\T}\Omega=0$ eliminates the coexact component but does not eliminate $H$. It therefore does not establish a global potential when the complex has unfilled cycles.

\subsection{Weighted fields and topology}

Confidence weighting can be expressed with positive-definite diagonal $W_k$ and $\langle u,v\rangle_{W_k}=u^{\T}W_kv$. The weighted adjoints and Laplacian are
\begin{equation}\label{eq:weighted}
  \delta_1^W = W_0^{-1}B_1W_1,\qquad
  \delta_2^W = W_1^{-1}B_2W_2,\qquad
  L_1^W = B_1^{\T}W_0^{-1}B_1W_1 + W_1^{-1}B_2W_2B_2^{\T}.
\end{equation}
Here $C^1=\im d_0\oplus_{W_1}\im\delta_2^W\oplus_{W_1}\ker L_1^W$. Orthogonality and energy use the $W_1$ inner product. Keep candidate face weights strictly positive during differentiation if the topology is to remain fixed. Setting a weight exactly to zero changes rank and may change $\beta_1$, while a small positive weight changes the metric but not the cohomology dimension. The choice of a face represents a substantive compositional assumption and must be ablated.

\section{HodgePhaseNet architecture and attention}

\subsection{Building the relation complex}

At a selected layer and attention head, choose a bounded candidate edge set from sequential neighbors, type-compatible knowledge and evidence links, and a limited number of pairs nominated by the \emph{base} attention scores. An edge connects two distinct vertices; the canonical orientation may be assigned by vertex index. Include a face only when all three edges exist and a stated relation composition rule admits it. This rule must be established using available input information rather than the labels of the held-out answer. The complex may be shared between heads, while $\Omega^{(\ell,h)}\in\R^{n_1\times r}$ and its decomposition are head-specific.

The relation encoder receives the Wave, Quantum, Modular, and evidence features of Section~\ref{sec:encoder}. It may be trained with relation direction and type supervision. A one-time edge selection from base attention avoids a circular instruction to select edges using attention logits that have not yet been computed. Discrete top-$k$ membership is generally nondifferentiable; training can hold membership fixed within a forward pass or use a fixed candidate graph with positive continuous weights. If there is no filled face, $C=0$ by construction; if the complex is a tree, both $C$ and $H$ are zero. These outcomes must be reported rather than interpreted as evidence that the model has detected no contradiction. Figure~\ref{fig:pipeline} locates the Hodge layer and its ordered-pair correction in the attention computation.

\begin{figure}[t]
  \centering
  \includegraphics[width=0.95\linewidth]{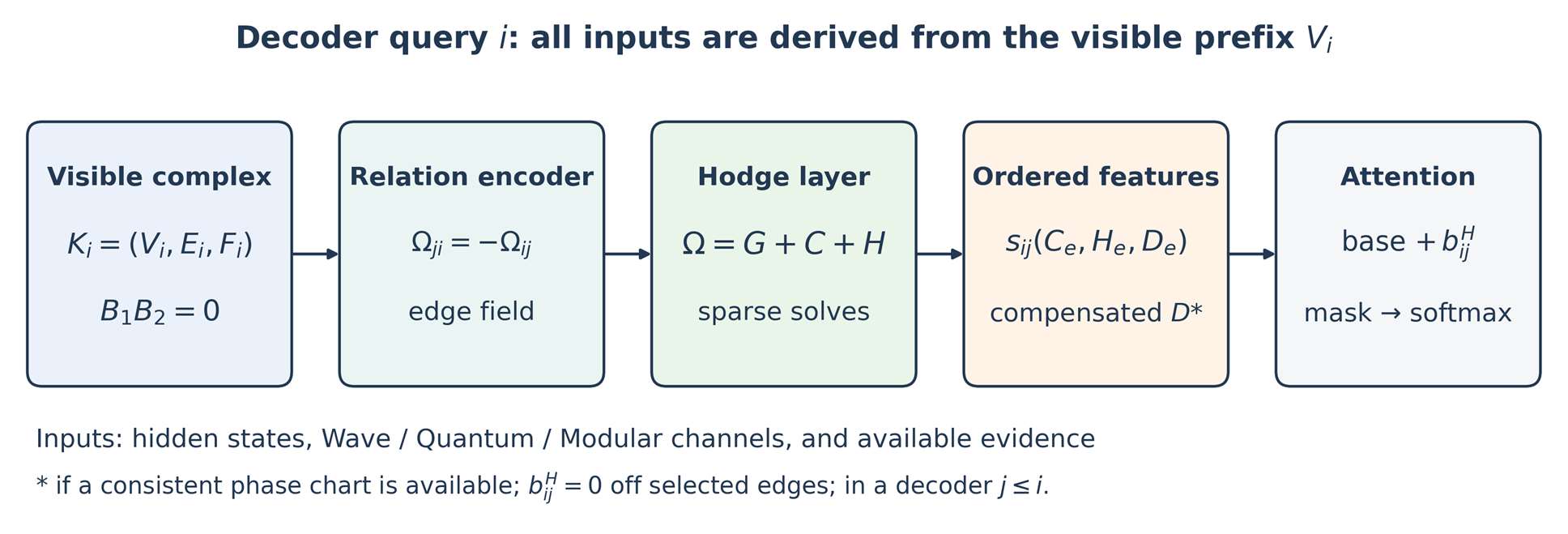}
  \caption{One-head computation path. All decoder features and the complex are constructed from the visible prefix. Edge orientation is converted for each ordered query--key pair; the compensated gradient $D$ is used only when its gauge assumptions hold. The learned edge bias precedes the causal mask and softmax.}
  \label{fig:pipeline}
\end{figure}

\subsection{Hodge layer and feature construction}

For each chosen head, compute the three fields by~\eqref{eq:solve}. A useful local feature on edge $e$ is the pair of \emph{oriented} coexact and harmonic values $(C_e,H_e)$, together with their orientation-independent norms. For a face $f$, let $\kappa_f=(B_2^{\T}\Omega)_f$; this is a signed local circulation. To associate faces with an edge, one may use
\begin{equation}\label{eq:facefeat}
  t_e = \frac{\sum_{f\colon e\subset\partial f}(B_2)_{ef}\,\kappa_f}{\max\bigl\{1,\sum_f\lvert(B_2)_{ef}\rvert\bigr\}},\qquad
  \sigma_s = \frac{\norm{s}_F^2}{\norm{C}_F^2+\norm{H}_F^2+\varepsilon},\qquad
  s\in\{C,H\},\ \varepsilon>0.
\end{equation}
The product $(B_2)_{ef}\kappa_f$ transforms like an oriented edge even if the face orientation is reversed. Norms alone cannot convey a signed semantic direction, whereas a signed edge feature without its orientation conversion depends on an arbitrary incidence convention.

The raw gradient $G=P_G\Omega$ changes under a gauge transformation of a connection. If token phases $\theta\in C^0(K;\R^r)$ are available in one consistent real chart and transform as $\theta'=\theta+\chi$, define the compensated gradient
\begin{equation}\label{eq:compensated}
  D_e = G_e - (B_1^{\T}\theta)_e.
\end{equation}
Then $D$ is gauge invariant under the convention of~\eqref{eq:gauge}. Alternatively, decompose a separately learned gauge-invariant relation field $\tilde{R}$ and use $P_G\tilde{R}$ as the directed hierarchy feature. If neither construction is available, omit the gauge-dependent $G$ from any bias advertised as gauge invariant. Taking $\norm{G_e}$ or $\rho_G$ from the transformed connection does \textbf{not} cure this problem: those quantities can change under gauge transformations. To use a global gradient energy in the invariant route, use $\rho_G$ of the compensated or separately invariant field instead.

\subsection{Ordered query--key scores}

For query $i$ and key $j$ in a permitted set $\cA_i$, let $e(i,j)$ be their selected unoriented edge when present. Set $s_{ij}=+1$ if the stored orientation of $e(i,j)$ is $i\to j$, and $s_{ij}=-1$ otherwise. The attention interface uses \textbf{ordered} pair features such as
\begin{equation}\label{eq:pairfeat}
  u_{ij} = \bigl[s_{ij}C_e;\ s_{ij}H_e;\ s_{ij}D_e;\ s_{ij}t_e;\ \norm{C_e}_2;\ \norm{H_e}_2;\ \sigma_C;\ \sigma_H;\ f^E_{ij};\ r_{ij}\bigr].
\end{equation}
For multichannel fields, signed quantities in~\eqref{eq:pairfeat} are vectors; a small learned map gives a scalar correction. Gauge-neutral Wave and finite-group invariants may also enter the pair map. A phase feature must be covariant or invariant under the stated local gauge rule. Features from the reverse order of an asymmetric relation are not silently reused as if they were symmetric.

Let $q_i=W_Qx_i$, $k_j=W_Kx_j$, and $v_j=W_Vx_j$. Define a finite correction and apply the usual mask \emph{after} it:
\begin{equation}\label{eq:attention}
\begin{aligned}
  b^H_{ij} &= \begin{cases}\gamma_h\tanh\bigl(f_{\eta,h}(u_{ij})\bigr), & e(i,j)\in E,\\ 0, & e(i,j)\notin E,\end{cases}\\
  z^H_{ij} &= \frac{q_i^{\T}k_j}{\sqrt{d_k}} + b^H_{ij} + m_{ij},\qquad
  m_{ij} = \begin{cases}0, & j\in\cA_i,\\ -\infty, & j\notin\cA_i,\end{cases}\\
  \alpha_{ij} &= \frac{\exp z^H_{ij}}{\sum_{t\in\cA_i}\exp z^H_{it}},\qquad
  y_i = \sum_{j\in\cA_i}\alpha_{ij}v_j.
\end{aligned}
\end{equation}
Each $\cA_i$ contains at least one key, usually $i$ itself. Self pairs have $b^H_{ii}=0$. A bounded learnable scale $\gamma_h$ and normalized features can limit large corrections. Setting $\gamma_h=0$ exactly recovers the base attention layer. The Hodge edge set supplies \emph{biases} on selected pairs; it need not replace the Transformer's dense allowed attention pattern. If one also sparsifies the attention mask, that is a separate intervention and must be compared separately. In particular, $b^H_{ij}$ is learned from task evidence: neither positive curl nor harmonic magnitude has a universal negative sign.

\begin{proposition}[orientation-independent logits]\label{prop:orientation}
If an edge orientation is flipped and its incidence matrices and edge field are transformed consistently, the bias in~\eqref{eq:attention} constructed from~\eqref{eq:pairfeat} is unchanged for every ordered pair.
\end{proposition}
\begin{proof}
Write edge and face sign changes as diagonal matrices $S_1,S_2$, with $B_1'=B_1S_1$, $B_2'=S_1B_2S_2$, and $\Omega'=S_1\Omega$. Equation~\eqref{eq:proj} gives $G'=S_1G$, $C'=S_1C$, $H'=S_1H$. The oriented face aggregate and compensated feature also change by $S_1$, while the sign $s_{ij}$ flips on a reversed edge. Each signed entry in~\eqref{eq:pairfeat}, every norm, and the resultant $b^H_{ij}$ are unchanged.
\end{proof}

\begin{proposition}[conditional gauge invariance of attention]\label{prop:gauge}
Suppose $\Omega'=\Omega+B_1^{\T}\chi$, $\theta'=\theta+\chi$, and all other inputs to $f_{\eta,h}$ and the base attention logit are gauge invariant. Then~\eqref{eq:pairfeat}, \eqref{eq:attention}, and $y_i$ are gauge invariant.
\end{proposition}
\begin{proof}
Orthogonality gives $P_C\Omega'=C$, $P_H\Omega'=H$, and $P_G\Omega'=G+B_1^{\T}\chi$. Equation~\eqref{eq:compensated} therefore gives $D'=D$; $B_2^{\T}\Omega'=B_2^{\T}\Omega$ by~\eqref{eq:chain}. The fractions $\sigma_C,\sigma_H$ are unchanged because their component fields are unchanged. The remaining inputs were assumed invariant. Softmax and weighted summation preserve equality.
\end{proof}

The proposition has real assumptions. A logit correction that uses raw $G$ or $\norm{G}$ fails its gauge claim in general. The same applies if the base query, key, value, or evidence branch changes under a phase gauge without compensation. Finite cyclic exponents and a wrapped real connection also require separate handling across branch cuts.

\subsection{Causality in an autoregressive Transformer}

A standard causal mask alone is insufficient to prevent leakage if a Hodge field for an early query was computed using later tokens. For a decoder query at position $i$, define the visible prefix $V_i=\{1,\dots,i\}$ and construct $K_i=(V_i,E_i,F_i)$ exclusively from states and evidence available by that position. Every edge endpoint, face, Wave feature, phase feature, modular feature, and relation label used to compute $b^H_{ij}$ must be measurable from that prefix. Then
\begin{equation}\label{eq:prefix}
  E_i\subseteq\binom{V_i}{2},\qquad
  F_i\subseteq\Bigl\{(a,b,c)\in\binom{V_i}{3} : \partial(a,b,c)\subseteq E_i\Bigr\},\qquad j\le i.
\end{equation}
During teacher-forced training, the Hodge features for position $i$ must be those of $K_i$, not projections on the full sequence subsequently hidden by the attention mask. At inference, a prefix solve or a validated incremental update may be reused as the visible graph grows. This requirement raises the cost of a literal per-token implementation; batching, selected layers, coarser claim-level graphs, and approximate incremental updates should be reported with their actual approximation error. A bidirectional encoder may use its full observed input; cross-attention may use a fully observed source, while generated target tokens remain prefix-restricted.

\subsection{Learning and numerical differentiation}

For an exact Hodge layer on a fixed complex, the decomposition's reconstruction and orthogonality losses vanish up to solver error; they should be diagnostics, not claimed learning signals. One possible training objective is
\begin{equation}\label{eq:loss}
  \cL = \cL_{\mathrm{task}} + \lambda_r\cL_{\mathrm{relation}} + \lambda_q\cL_{\mathrm{phase}} + \lambda_m\cL_{\mathrm{holonomy}} + \lambda_c\cL_{\mathrm{calibration}} + \lambda_b\cL_{\mathrm{budget}}.
\end{equation}
Relation supervision identifies direction or type; the phase and finite-group objectives apply where labeled or self-supervised cycles are available; the budget term controls edges, faces, FLOPs, and memory. A blanket penalty on $\norm{C}$ or $\norm{H}$ would erase legitimate cycles and is not imposed. Differentiation through a converged sparse solve may use implicit differentiation on gauge-fixed systems; a fixed number of unrolled iterations gives a differentiable approximation with tolerance-dependent error. When a discrete top-$k$ support is used, derivatives do not pass through its membership decision.

A fast amortized version may learn $\hat\Phi,\hat\Psi,\hat H$ and set $\hat G=B_1^{\T}\hat\Phi$, $\hat C=B_2\hat\Psi$. Distill against exact fields and measure
\begin{equation}\label{eq:distill}
\begin{aligned}
  \cL_{\mathrm{dist}} &= \sum_{s\in\{G,C,H\}}\norm{\hat s - P_s\Omega}_F^2,\\
  \cL_{\mathrm{rec}} &= \norm{\Omega-\hat G-\hat C-\hat H}_F^2,\\
  \cL_{\mathrm{harm}} &= \norm{B_1\hat H}_F^2 + \norm{B_2^{\T}\hat H}_F^2.
\end{aligned}
\end{equation}
Parameterizing $\hat G$ and $\hat C$ through the incidence maps is important: merely requiring $B_2^{\T}\hat G=0$ does not rule out a harmonic contribution, and merely requiring $B_1\hat C=0$ does not force it to be coexact. The approximate variant has no exact decomposition guarantee until its measured errors justify one.

\subsection{Calibrated error risk and complexity}

For a defined future label of unsupported or incorrect output, aggregate visible-prefix features such as $z_t=\bigl(\rho_C,\ \rho_H,\ \norm{B_2^{\T}\Omega}_F^2,\ \Hol_q,\ \text{evidence support},\ \text{semantic uncertainty}\bigr)$. A learned score $s_t=f_\beta(z_t)$ is mapped to an estimated probability $\hat p_t=\Calibrate(s_t)$ using a separate calibration set. Temperature scaling or isotonic regression can be assessed with Brier score and negative log-likelihood. If conformal risk control is used, its coverage statement must be specified separately; it does not turn a raw score into a calibrated probability. The score is an empirical prediction for its labeled endpoint, not a truth oracle.

For one complex with $r$ channels and $T$ sparse iterations, matrix multiplication costs roughly
\begin{equation}\label{eq:cost}
  O\bigl(Tr\,[\nnz(B_1)+\nnz(B_2)]\bigr)
\end{equation}
plus relation encoding, complex construction, and ordinary attention. A triangular complex has $\nnz(B_2)=3n_2$. Selecting $O(kn_0)$ edges can keep the \emph{Hodge solve} sparse but leaves standard dense attention quadratic unless that attention is also altered. A separate exact solve for all $n_0$ causal prefixes multiplies the effective work; \eqref{eq:cost} applies to a single complex, not automatically to an entire generated sequence. Report wall time and peak memory, not just asymptotic sparse multiplication.

\section{Further mathematical properties}

\subsection{Gauge behavior, finite holonomy, and stability}

For a fixed real phase chart and a connection field obeying~\eqref{eq:gauge}, the proof of Proposition~\ref{prop:gauge} shows
\begin{equation}\label{eq:gaugeproj}
  P_C(\Omega+B_1^{\T}\chi)=P_C\Omega,\qquad
  P_H(\Omega+B_1^{\T}\chi)=P_H\Omega,\qquad
  P_G(\Omega+B_1^{\T}\chi)=G+B_1^{\T}\chi.
\end{equation}
These are statements about the decomposition, not an automatic guarantee that every learned semantic feature is gauge invariant. Finite group holonomy remains in its exact modular channel. A principal real representative such as
\begin{equation}\label{eq:wrap}
  \tilde a_e = \wrap_{(-1/2,\,1/2]}\!\left(\frac{a_e}{q}\right)
\end{equation}
can enter a real relation channel, but crossing a wrap boundary creates a discontinuity and changes which real lift is used. There is no canonical positive-definite real inner product on $\Z/q\Z$ that makes its group cohomology identical to the real orthogonal splitting of~\eqref{eq:hodge}. Evaluate modular holonomy and real component energies side by side.

For a fixed unweighted complex, each orthogonal projector is nonexpansive:
\begin{equation}\label{eq:nonexp}
  \norm{P_s(\Omega+\Delta)-P_s\Omega}_F = \norm{P_s\Delta}_F \le \norm{\Delta}_F,\qquad s\in\{G,C,H\}.
\end{equation}
This does not imply stability when $E$ or $F$ changes, a solver is truncated, or the nonlinear model is trained. It also does not establish global Lipschitz continuity of the complete Transformer.

\subsection{Attention output bound}

\begin{theorem}[convex combination bound]\label{thm:bound}
For any row of~\eqref{eq:attention} with finite allowed logits and $\norm{v_j}_2\le B_V$ for all permitted keys, $\norm{y_i}_2\le B_V$.
\end{theorem}
\begin{proof}
The softmax weights are nonnegative and sum to one, so $y_i$ is a convex combination of the permitted value vectors. Apply the triangle inequality.
\end{proof}

This elementary bound is not a claim of improved training stability or semantic validity.

\section{Verified five-vertex example}

\subsection{A filled local cycle and an unfilled global cycle}

Take $V=\{1,2,3,4,5\}$ and oriented edges $e_1\colon1\to2$, $e_2\colon2\to3$, $e_3\colon3\to1$, $e_4\colon1\to4$, $e_5\colon4\to5$, $e_6\colon5\to1$. Fill the first triangle with a face and leave the second unfilled. Then
\begin{equation}\label{eq:example}
  B_1 = \begin{bmatrix}
    -1 & 0 & 1 & -1 & 0 & 1\\
    1 & -1 & 0 & 0 & 0 & 0\\
    0 & 1 & -1 & 0 & 0 & 0\\
    0 & 0 & 0 & 1 & -1 & 0\\
    0 & 0 & 0 & 0 & 1 & -1
  \end{bmatrix},\qquad
  B_2 = \begin{bmatrix}1\\1\\1\\0\\0\\0\end{bmatrix}.
\end{equation}
Direct multiplication yields $B_1B_2=0$. The graph is connected, $\rank B_1=4$, and $\rank B_2=1$, so $\beta_1=6-4-1=1$. Figure~\ref{fig:example} shows the orientations and the single filled face of this complex.

\begin{figure}[t]
  \centering
  \includegraphics[width=0.85\linewidth]{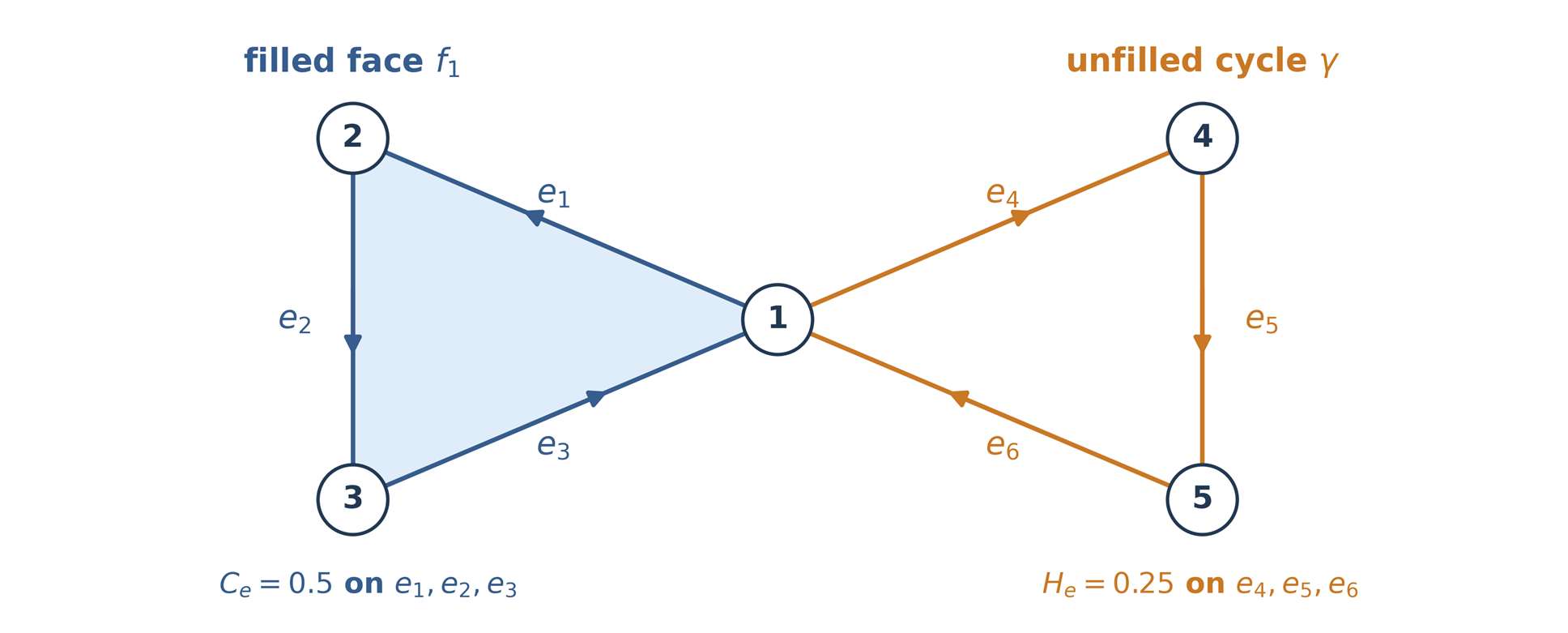}
  \caption{Five-vertex example of Section~7.1. The blue triangle $(1,2,3)$ is a filled face and carries the coexact field $C$; the orange triangle $(1,4,5)$ has no face and supports a harmonic cycle $H$. Arrows follow the edge orientations in~\eqref{eq:example}; the annotated values are component fields, not the full $\Omega$.}
  \label{fig:example}
\end{figure}

Choose
\begin{equation}
  \Phi=(0,1,2,-1,1)^{\T},\qquad \Psi=(0.5),\qquad H=(0,0,0,0.25,0.25,0.25)^{\T}.
\end{equation}
Then
\begin{equation}
\begin{aligned}
  G &= B_1^{\T}\Phi = (1,1,-2,-1,2,-1)^{\T},\\
  C &= B_2\Psi = (0.5,0.5,0.5,0,0,0)^{\T},\\
  \Omega &= G+C+H = (1.5,1.5,-1.5,-0.75,2.25,-0.75)^{\T}.
\end{aligned}
\end{equation}
The harmonic vector satisfies $B_1H=B_2^{\T}H=0$. Pairwise inner products of $G,C,H$ vanish, and their squared norms are $12$, $0.75$, and $0.1875$ respectively. Thus $\norm{\Omega}_2^2=12.9375$. Direct numerical projection with pseudoinverses reproduces each field to floating-point precision. The filled first triangle supports local coexact circulation; the second, unfilled triangle supports the remaining harmonic cycle. This example establishes an algebraic distinction, not a semantic verdict about either cycle.

\section{Implementation protocol}

\subsection{One selected layer and head}

\begin{enumerate}[leftmargin=*]
  \item Obtain visible hidden states, relation metadata, available evidence, and the Wave, Quantum, and Modular feature channels. In a decoder, first restrict every channel to the query's visible prefix.
  \item Select vertices and a bounded candidate edge set using a fixed rule or the base attention logits. Assign canonical edge orientations. Add a face only when its three edges exist and a specified composition rule allows it. Check $B_1B_2=0$.
  \item Compute both ordered relation-encoder outputs for each edge and form the antisymmetric $\Omega$ by~\eqref{eq:encoder}. If a gauge claim is made, enforce the connection-plus-invariant-residual parameterization.
  \item Solve the two sparse systems in~\eqref{eq:solve}, recover $G,C,H$, and record reconstruction, orthogonality, divergence, and face-circulation residuals. Do not compute a full harmonic eigenbasis for routine inference.
  \item For each permitted ordered query--key pair, multiply the signed edge values by $s_{ij}$, form~\eqref{eq:pairfeat}, and calculate a finite scalar bias. Use zero bias on pairs outside the relation complex and on self pairs; apply the original attention mask and softmax by~\eqref{eq:attention}.
  \item Compute the task objective, optional relation and phase objectives, and resource penalties. Train a separate risk calibrator on held-out calibration data after choosing the base model and score.
\end{enumerate}

Three deployment variants should be compared: a converged sparse solve for an accuracy reference; a specified fixed-iteration solver with reported residuals; and a distilled approximate head with the losses in~\eqref{eq:distill}. Approximation speed is an empirical measurement, not implied by the mathematical decomposition.

\subsection{Runtime invariance checks}

Randomly flip edge and face orientations and verify that (i) the three fields transform by the same edge sign matrix; (ii) their energies are unchanged; and (iii) the ordered logits and outputs of~\eqref{eq:attention} are unchanged. For gauge tests, add $B_1^{\T}\chi$ to an appropriate real connection, transform $\theta$ in the same chart, and compare $C$, $H$, $D$, bias, and output. Check an explicit future-token intervention: editing tokens after position $i$ must leave all logits and outputs for queries through $i$ unchanged in a causal decoder. Finally verify that setting $\gamma_h=0$ recovers the original masked attention exactly up to numerical precision.

\section{Empirical evaluation plan}

\subsection{Questions and datasets}

The evaluation separates mathematical recovery from semantic hypotheses. In a synthetic complex with known $G,C,H$, measure relative error of each field, the number of independent harmonic cycles, projector orthogonality, and sensitivity to missing edges, faces, noise, and solver tolerance. A reported synthetic recovery is a test of the implementation, not evidence of natural-language validity.

For hierarchy and direction, use lexical hierarchy resources such as WordNet and graded lexical entailment data such as HyperLex~\cite{miller1995,vulic2017}, with relation direction accuracy and rank correlation. For local consistency, construct labeled three-claim compositions and compare with fact verification and NLI sources such as FEVER and ANLI~\cite{thorne2018,nie2020}. Neither dataset alone supplies a canonical simplicial complex or a labeled cycle; the construction protocol and labels for valid versus contradictory cycles must be released. For global discourse, include long claim chains and document-order perturbations that are locally plausible but globally inconsistent, alongside legitimate topic returns. For generated factual claims, evaluate source-supported atomic facts and selective answering on TruthfulQA and FActScore-style settings~\cite{lin2022,min2023}. Separately report Brier score, negative log-likelihood, expected calibration error with its binning rule, and risk--coverage curves~\cite{guo2017}. Compare against semantic uncertainty and retrieval-based controls~\cite{farquhar2024,lewis2020}.

\begin{center}
\small
\begin{tabular}{@{}lll@{}}
\toprule
Mechanism & Essential outcome & Critical control\\
\midrule
Algebraic recovery & Component error and orthogonality & Known fields and changed topology\\
Hierarchy & Direction and graded entailment & Equal-size relation encoder\\
Local cycles & Valid cycle versus contradiction AUPRC & Raw face sum and modular holonomy\\
Global cycles & Legitimate return versus inconsistency & Local consistency and graph size\\
Factuality & Supported atomic facts and selective risk & Same evidence and inference budget\\
Calibration & Brier score and risk--coverage & Separate calibration split\\
\bottomrule
\end{tabular}
\end{center}

All English and Japanese analyses require language-appropriate resources and separate calibration; an English threshold should not be assumed transferable. Hold out documents, entities, or claim groups across training, calibration, and test whenever related cycles share underlying evidence. Distinguish a missing retrieved source from an actually false claim.

\subsection{Matched baselines and ablations}

Match tokenizer, backbone, training tokens, evidence availability, maximum candidate edges, and inference budget. Compare: an unmodified Transformer; each available Wave, Quantum, and Modular feature branch alone~\cite{wavephasenet,quantumphasenet,modularphasenet}; a parameter-matched pair-relation MLP without Hodge projection; vertex graph-Laplacian smoothing; a representative simplicial message-passing model~\cite{ebli2020,bodnar2021}; and the exact, fixed-iteration, and distilled Hodge versions. A stronger evidence-retrieval or verifier baseline is needed before assigning a factuality gain to topology.

Required ablations replace $\Omega$ with $B_1^{\T}X$; remove all faces; fill all admissible triangles; use only one of $G,C,H$ at a time; remove each PhaseNet feature family, evidence features, or the risk calibrator; change the edge and face budget; permute relation labels; randomly reverse edge orientations; transform the phase gauge in a controlled chart; and compare prefix-specific causal projections with a deliberately full-sequence control to expose leakage. Changing face coverage can shift energy between $C$ and $H$ without changing the text, so semantic claims must be robust to reasonable complex-construction choices.

Use several independent seeds and report means, variability, confidence intervals on paired examples or documents, and the multiplicity correction for predeclared primary comparisons. For rare contradictions, AUPRC is essential alongside AUROC. Report added trainable parameters, floating-point operations, sparse solve iterations, peak memory, and measured generation latency; keep model selection, calibration, and final assessment separate.

\subsection{Criteria that would narrow the proposal}

The semantic interpretation should be reduced or rejected if $C,H$ cannot discriminate valid from problematic cycles beyond matched relation and triangle-sum controls; if face construction changes reverse the findings under plausible rules; if an equal-cost relation MLP gives the same gain; if orientation or gauge tests fail their stated assumptions; or if factuality gains disappear after matching evidence and compute. A risk score that does not improve Brier score or selective risk over entropy and evidence controls cannot support a calibration claim. Failure to beat a cheaper verifier on quality per latency would constrain the practical case for a Hodge layer even if the decomposition is mathematically correct.

\section{Discussion and limitations}

An exact component is a potential difference and can encode a consistent direction after a learned orientation convention, but a perfectly consistent false account can also be exact. A coexact component records circulation explained by the selected local faces; it can mark a violated transitive relation or a legitimate reciprocal definition. A harmonic component records divergence-free flow with zero circulation on declared faces; it can reflect an unsupported reasoning loop or an intentional long-range return. Meaning attaches to typed edges, face rules, evidence, and supervised outcomes, not to the names of the three subspaces.

The decomposition depends on vertex granularity, edge selection, and face filling. A dynamic threshold can change topology discontinuously. A wrapped phase lift is chart-dependent. Gauge invariance of $C,H$ does not make a raw gradient-based attention branch invariant. Exact causal decoding may require many solves; incremental updates and learned approximations introduce measurable errors. The mathematical uniqueness result applies to fields on a fixed complex, not to the choice of complex or to an induced interpretation of language.

No natural-language improvement is reported here. Results previously reported for QuantumPhaseNet's built-in synthetic validation~\cite{quantumphasenet} are not experiments on HodgePhaseNet and are not reused as evidence for it. In high-consequence domains, a calibrated risk model may help prioritize verification or abstention, but the output should retain source review and human oversight. Potential resource and cultural bias in knowledge links, relation annotations, and retrieval sources should be measured across languages and viewpoints.

\section{Conclusion}

HodgePhaseNet gives a specific mathematical place for the PhaseNet sequence of ideas to meet: spectral features and covariant or finite phase relations inform an antisymmetric edge field, whose exact, coexact, and harmonic projections separate three kinds of organization. The connection from this field back to Transformer attention requires ordered-pair signs, a visible prefix for causal decoding, a gauge-aware choice of gradient feature, and explicit solver and topology checks. The algebra and the five-vertex example are exact; proposed semantic benefits and risk calibration remain testable hypotheses. The next empirical step is an audited sparse reference implementation with synthetic recovery, matched natural-language baselines, and independently labeled valid and invalid cycles.

\appendix

\section{Projection derivation}

The gradient projection solves
\begin{equation}
  \min_{\Phi}\ \tfrac12\norm{\Omega-B_1^{\T}\Phi}_F^2 \quad\Longrightarrow\quad B_1B_1^{\T}\Phi = B_1\Omega. \tag{A1}
\end{equation}
Choosing the minimum-norm solution gives $\Phi=(B_1B_1^{\T})^{\dagger}B_1\Omega$ and the first projector in~\eqref{eq:proj}. Likewise,
\begin{equation}
  \min_{\Psi}\ \tfrac12\norm{\Omega-B_2\Psi}_F^2 \quad\Longrightarrow\quad B_2^{\T}B_2\Psi = B_2^{\T}\Omega. \tag{A2}
\end{equation}
The spaces are orthogonal by $B_1B_2=0$, so the two least-squares problems can be solved independently. The residual $(I-P_G-P_C)\Omega$ is annihilated by both $B_1$ and $B_2^{\T}$. An orientation change $S_1,S_2$ conjugates the edge projectors, $P_s'=S_1P_sS_1$, which supplies the algebra used in Proposition~\ref{prop:orientation}.

\section{Notation and reproducibility checks}

\begin{center}
\begin{tabular}{@{}ll@{}}
\toprule
Symbol & Meaning\\
\midrule
$K=(V,E,F)$ & Finite oriented semantic complex\\
$B_1,\ B_2$ & Vertex--edge and edge--face incidence matrices\\
$\Phi,\ \Omega,\ \Psi$ & Vertex potential, edge field, face potential\\
$G,\ C,\ H$ & Exact, coexact, and harmonic edge fields\\
$P_G,\ P_C,\ P_H$ & Orthogonal projectors for a fixed unweighted complex\\
$\beta_1$ & First Betti number\\
$\rho_s$ & Squared energy fraction of component $s$\\
$a_e,\ \Hol_q$ & Finite cyclic connection and cycle holonomy\\
$s_{ij},\ D_e$ & Ordered-pair orientation sign and compensated gradient\\
$\cA_i$ & Keys visible to query $i$\\
\bottomrule
\end{tabular}
\end{center}

An implementation should record the construction of $E$, $F$, $B_1B_2$, solver residuals, reconstruction and cross-component inner products, $B_1H$, $B_2^{\T}H$, orientation and gauge tests, and the base attention recovery at zero Hodge correction. The scalar five-vertex test in Section~7 provides exact reference values. For causal generation, a suffix intervention test is necessary in addition to checking the mask matrix.

\end{document}